\documentclass{esannV2}
\usepackage[dvips]{graphicx}
\usepackage[utf8]{inputenc}
\usepackage{amssymb,amsmath,array,bm,amsthm,mathtools,booktabs}

\DeclareMathOperator*{\argmin}{argmin}

\renewcommand{\figurename}{Figure}

\theoremstyle{definition}
\newtheorem{theorem}{Theorem}

\newtheorem{Definition}[theorem]{Definition}

\newtheorem{Lemma}[theorem]{Lemma}
\newtheorem{Assumption}[theorem]{Assumption}

\begin{document}

\title{Constraint Guided Gradient Descent: Guided Training with Inequality Constraints}

\author{Quinten Van Baelen$^{1,2,3}$ and Peter Karsmakers$^{1,2,3}$

\thanks{This research received funding from the Flemish Government (AI ResearchProgram). This research has received support of Flanders Make.}

\vspace{.3cm}\\

1-KU Leuven, Dept. of Computer Science, ADVISE-DTAI, \\ Kleinhoefstraat 4, B-2440 Geel, Belgium. \\

2-Leuven.AI - KU Leuven institute for AI. \\
3-Flanders Make - DTAI-FET.\\
\texttt{\{quinten.vanbaelen,peter.karsmakers\}@kuleuven.be}
}

\maketitle

\begin{abstract}
Deep learning is typically performed by learning a neural network solely from data in the form of input-output pairs ignoring available domain knowledge. In this work, the Constraint Guided Gradient Descent (\texttt{CGGD}) framework is proposed that enables the injection of domain knowledge into the training procedure.  The domain knowledge is assumed to be described as a conjunction of hard inequality constraints which appears to be a natural choice for several applications. Compared to other neuro-symbolic approaches, the proposed method converges to a model that satisfies any inequality constraint on the training data and does not require to first transform the constraints into some ad-hoc term that is added to the learning (optimisation) objective. Under certain conditions, it is shown that \texttt{CGGD} can converges to a model that satisfies the constraints on the training set, while prior work does not necessarily converge to such a model. It is empirically shown on two independent and small data sets that \texttt{CGGD} makes training less dependent on the initialisation of the network and improves the constraint satisfiability on all data.  
\end{abstract}

\section{Introduction}

Machine learning and especially deep learning are successful in many research areas. In most cases, supervised learning is employed that, based on example input-output pairs, automatically finds a function that relates the input to the corresponding output data. However, available domain knowledge is typically ignored requiring it to be rediscovered by the learning algorithm. When domain knowledge can be inserted during the learning stage, it is expected that learning becomes more efficient, meaning that less example pairs are required to let a model represent the desired relation.

This study restricts itself to the use of a conjunction of hard inequality constraints. Hence, models should satisfy all imposed inequality constraints for all the data (even for unseen data, not used during learning, the model should satisfy the constraints). This work proposes a novel algorithm Constraint Guided Gradient Descent (\texttt{CGGD}), which adds supervision to the learning cycle by means of hard inequality constraints. \texttt{CGGD} aims at solving the potential numerical problems and the crispness issues that occur in previous work.  Moreover, in \texttt{CGGD} the constraints do not need to be differentiable, and they provably dominate the gradient of the loss function during training when they are not satisfied.

There are two main classes of approaches that enable injecting constraints in the training procedure. The first class of approaches uses fuzzy-logic \cite{DL2TrainingAndQuerying,Bach2017}. Here, the constraints are replaced by almost everywhere smooth functions. As said in \cite{SemanticLoss}, this approach has as its main downside that this transformation typically leads to a loss of the crisp formulation of the constraints. Additionally, there can occur numerical problems when optimising the new objective. For example, the gradients of the loss function and the constraints can cancel out each other. However, \texttt{CGGD} solves both the crispness issue as well as the vanishing gradient phenomenon.

The second class of approaches can be summarised as using (probabilistic) logic reasoning in order to define gradients for training the network and/or as regularisation. The constraints in this setting are logical formulas, where the variables in the formulas correspond to Boolean, probabilistic or discrete output variables of the network. The methods that are most related to \texttt{CGGD} are: \texttt{NeuroLog} \cite{Tsamoura2021}, \texttt{DeepProbLog} \cite{Manhaeve2021}, and the semantic loss \cite{SemanticLoss}. Each method does not require the theory to be differentiable, but uses results from reasoning on the logic theory to tune the gradient with which the network is updated. All three methods are not applicable in the setup of this work because adjusting the truth value of an inequality constraint requires an additional reasoning mechanism. 

The main contributions of this work are: (a) the design of the novel \texttt{CGGD} method that learns a neural network model for a regression task while satisfying a conjunction of hard inequality constraints, (b) the empirical observation that incorporating prior knowledge in terms of inequality constraints can make learning less dependent on the initialisation of the model parameters.

\section{Constraint Guided Gradient Descent}
\label{sec:CGGD}

This work targets an algorithm that searches for the weights of a neural network $\Phi$ by optimising some loss function $L$ while letting the network satisfy a fixed finite set of predefined hard inequality constraints $\{C_i\}_{i=1}^N$ on the training set. More formally, this can be expressed as the constrained optimisation problem:
	\begin{align}
		\nonumber
		& \argmin \limits_{\bm{W}} && L(\bm{x}, \Phi(\bm{x}), \bm{y}, \bm{W}) \\ \nonumber
		& \,\,\,\,\, \text{s.t.} && \forall (x,\Phi (x)) \in (\bm{x}, \Phi(\bm{x})): C_i (x, \Phi(x)) \leq 0, \mbox{ for } i=1,\ldots,N.
	\end{align}
In the previous equation, $\bm{x}$ and $\bm{y}$ denote a set of input vectors and output vectors respectively, $x$ and $y$ denote a single input vector and output vector respectively, $\Phi(x)$ denotes the predictions of the network as well as any prediction of any hidden layer, and $\bm{W}$ denotes the collection of trainable weight matrices of the model. The set of models that satisfy all constraints for a set of training examples is called the \textbf{feasible region} $FR$. \texttt{CGGD} aims at finding a model in $FR$ that locally minimises $L$.

The constrained optimisation problem is solved by optimising the loss function with gradient descent and adjusting the update step according to the constraints in case they are not satisfied. When some constraints are not satisfied, then for each unsatisfied constraint a direction is computed to move to in order to satisfy the constraint eventually. Hence, the update step for a trainable parameter $w$ is defined by
    \begin{equation} \label{def:UpdateStep} 
    	w^{(i+1)}:= \,w^{(i)} - \eta_i (\nabla L(\Phi(x))   + 1.5 \  \, \overset{\rightarrow}{dir}(C(x,\Phi(x))) \max\{\varepsilon,\|\nabla L(\Phi(x)) \|\}), 
    \end{equation}
where $\eta_i$ denotes the step size for iteration $i$, $\overset{\rightarrow}{dir}(C(x,\Phi(x)))$ denotes the direction corresponding to the constraints $C:=\{C_i\}_{i=1}^N$, $1.5$ is a factor that is referred to the rescale factor that controls the relative weight of the constraints compared to the gradient of the loss function, $\|\cdot\|$ denotes the $L_2$-norm, and 
$\varepsilon>0$ is a lower bound for the relative weight compared to the gradient of the loss function to allow to move past local optima outside $FR$. Note that the proposed update step \eqref{def:UpdateStep} does not introduce a new hyperparameter that needs to be chosen correctly, and the rescale factor is set larger than 1, which is equivalent with the constraints being more important than the loss function.

The following assumption is needed to guarantee convergence when the constraints do not have any influence on the optimisation procedure at some point in time and onwards, for example when the initialisation and every point in the optimisation procedure are in $FR$. 

\begin{Assumption}\label{axiom:NonConvexOptimal}
    Let $L:\mathbb{R}^n\to\mathbb{R}$ satisfy conditions needed to let a non-convex optimisation algorithm based on gradient descent converge to a local solution.
\end{Assumption}

The main result of this paper is stated now. Note that all conditions that are stated, with the exception of Assumption  \ref{axiom:NonConvexOptimal}, are used for proving convergence to a point on the boundary of $FR$, which is in $FR$ when it is closed. 

\begin{theorem}\label{thm:ConvergenceTheorem}
    Let $L:\mathbb{R}^n\to\mathbb{R}$ be a loss function satisfying Assumption \ref{axiom:NonConvexOptimal} and for which $\nabla L$ is $M$-Lipschitz continuous. Consider the inequality constraints $\{C_i\}_{i=1}^N$ for some strictly positive integer $N$. Let $\overset{\rightarrow}{dir}(C(x,\Phi(x)))$ be the direction of the shortest path with respect to the Euclidean distance from $FR$ to $w$ for $w\in \mathbb{R}^{n}\setminus FR$. Then, there exists a sequence $\{\eta_j\}_j$ such that the iteration procedure defined by applying \eqref{def:UpdateStep} converges to a point in the closure of $FR$.
\end{theorem}

The proof of this theorem (Appendix A) consists of (i) showing that the size of the update step can be decreased over different iterations by decreasing $\eta_i$ as a function of $\varepsilon$ and $\|\nabla L\|$, and (ii) showing that the point obtained from one iteration is closer to the feasible region than the previous point. The direction of the constraints being defined by the shortest path to $FR$ is a sufficient condition but not a necessary condition. For example, if $FR=[1,2]\cup[3,4]$. Then the direction of the constraint can be chosen as $-1$ for $w<3$ and $1$ for $w>4$. This leads to \texttt{CGGD} converging to $w\in [3,4]$ when initialised at $w=2.1$. 

An example is given to illustrate the importance of Theorem \ref{thm:ConvergenceTheorem}. Let $L:\mathbb{R}\to\mathbb{R}: w \mapsto (w-2)(w-4)(w-3)(w-1.5)(w-1)(w-2.75)(w-5)^2+7$. Suppose that the constraint is given by $(w-1)(w-2)(w-3)(w-4)\leq 0$ for $w\in\mathbb{R}$. This leads to the feasible region being $[1,2]\cup[3,4]$. From the visualisation of $L$ in Figure \ref{Fig:derivatesfuzzyCGGD} is clear that the local minima are given by $w\approx 1.16$, $w=2$, $w\approx 3.63$. Moreover, Figure \ref{Fig:derivatesfuzzyCGGD} illustrates the value of the update steps for a fuzzy loss function, which adds the constraint as regularisation term to the learning objective before optimising with gradient descent, and \texttt{CGGD}. Note that a fuzzy approach requires the constraints to be almost everywhere differentiable, while this is not necessary for constraints in \texttt{CGGD}. From determining the points where the update step is equal to 0 or the sign of the update step is negative to the left and positive to the right of the point in case of a discontinuity, it follows that the fuzzy approach can converge to $w\approx1.16$, $w\approx1.71$ (when initialised at this point), $w\approx 2.27$, $w\approx 2.82$ (when initialised at this point), and $w=3.63$, while \texttt{CGGD} can converge to $w\approx1.16$, $w\approx 1.71$ (when initialised at this point), $w=2$ and $w\approx3.63$.  This illustrates the fact that the gradient of a fuzzy loss function can vanish even when constraints are not satisfied. While the points that can be obtained as convergence points of \texttt{CGGD} satisfy the constraints. 

\begin{figure}[t]
	\centering 
	\includegraphics[scale=0.34]{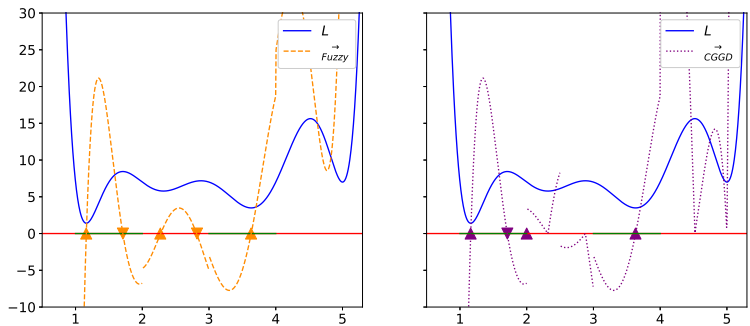}
	\caption{Loss function with the gradient of the fuzzy loss function (left) and the \texttt{CGGD} update step (right). The local solutions of the optimisation procedure are indicated with triangles on the horizontal axis. The triangles pointing upwards and downwards indicate if the convergence is stable or not, respectively. The convergence is not stable when it can only convergence if initialised at this point. The feasible region is shown in green on the horizontal axis.}
	\label{Fig:derivatesfuzzyCGGD}
\end{figure}

Another major difference with fuzzy approaches that optimises an objective function with gradient descent require almost everywhere differentiable constraints, while in \texttt{CGGD} the constraints can be non-differentiable for a set of strictly positive measure. For example, consider the constraint $-2\leq C(w)\leq 2$, where $C:\mathbb{R}\to\mathbb{R}:w\mapsto w \chi_{\mathbb{Q}}(w) - w \chi_{\mathbb{R}\setminus \mathbb{Q}}(w)$ with $\chi_A$ the indicator function on the set $A$. Observe that this function is only continuous in $w=0$. Therefore, it is not almost everywhere differentiable. Note that the direction of the shortest path for \texttt{CGGD} can be taken $-1$ if $w<-2$ and $1$ if $w>2$. 

\section{Experiments}
\label{sec:Experiments}

The presented method, \texttt{CGGD}, is tested for its performance compared to two baselines\footnote{See \texttt{https://github.com/KULeuvenADVISE/CGGD} for the code of the experiments.}. The size of the data sets is 750 examples. The division into training, validation and test set is 200/250/250. The first baseline (Baseline) is the model trained without any constraints. The second baseline (Fuzzy) is obtained using the loss function used in \texttt{DL2} \cite{DL2TrainingAndQuerying}. The training procedure discussed for \texttt{DL2} is not used, since it is not feasible to adjust it to the constraints considered here. Each setup is repeated 4 times with different initialisations of the network, and the mean and standard deviation of each metric are reported.

The first data set is the Bias Correction\footnote{\hbox{Available on }\texttt{https://archive.ics.uci.edu/ml/datasets/Bias+correction+of\\+numerical+prediction+model+temperature+forecast} \cite{Cho2020}.} (BC) data set. The task is to predict the maximal and minimal temperature of the next day given some information of the current day. The constraints considered for this data set are: upper and lower bounds on the values for both the minimal temperature and the maximal temperature, and the constraint that the minimal temperature should be smaller or equal than the maximal temperature. 

The second data set is the Family Income\footnote{Available on \texttt{https://www.kaggle.com/grosvenpaul/family-income-and-expenditure}.} (FI) data set. The task is to predict certain expenses of a family given information about the household income and some information about the properties owned by the household such as the number of personal computers. Also here for this data set, upper and lower bounds are set on all the predicted values. Moreover, the total food expenditure prediction should be larger than the sum of the prediction of the bread and cereals, the meat, and the vegetables expenditure. The last constraint is that the total income of the family (input) should be larger or equal than the sum of all the expenses.

While  \texttt{CGGD} can be more generically applied to different architectures, in this work, only dense neural networks are considered. The hidden layers have ReLU activation functions and the final layer has a linear activation function. The number of hidden layers are different for each data set and are only chosen such that all the constraints could be satisfied on the training set.

All networks are trained and tested using the Means Squared Error (MSE) as loss function. The satisfaction ratio (SR) is introduced as a metric to indicate how many constraints are satisfied. The satisfaction ratio is the ratio between the total number of satisfied constraints and the total number of constraints. 

\section{Results}
\label{sec:Results}

The results of the experiments are shown in Table \ref{tab:ResultsExperiment}. The experiments show that the proposed method has less problems with having a decent or good performance for small training sets compared to the other methods. In particular, the results indicate that \texttt{CGGD} seems to depend less on the initialisation of the network. This was also shown in the one-dimensional examples in Section 2. It is interesting to observe the same phenomena for neural networks as well in the experiments. The main reason for this is that loss functions of neural networks are known to be highly non-convex, which increases the likelihood of having a vanishing gradient in fuzzy approaches as illustrated in Section 2.

\begin{table*}[t]
	\centering
	\fontsize{9}{10} \selectfont
	\begin{tabular}{lcccc}
		\toprule
		 & \multicolumn{2}{c}{BC} & \multicolumn{2}{c}{FI} \\
		 \cmidrule(lr){2-3}\cmidrule(lr){4-5}
		Method & MSE & SR & MSE & SR \\
		\midrule
		Baseline & 0.7441$\pm$0.5593 & 74.90$\pm$10.21  & \textbf{0.0012$\pm$0.0001} & 98.69$\pm$0.22 \\
		Fuzzy & 0.0129$\pm$0.0049 & 99.52$\pm$\,\,\,0.48 & 0.0066$\pm$0.0020 & 95.87$\pm$0.26  \\
		\texttt{CGGD} & \textbf{0.0079$\pm$0.0084} & \textbf{99.96$\pm$\,\,\,0.05} & 0.0017$\pm$0.0005  & \textbf{99.89$\pm$0.12} \\
		\bottomrule
	\end{tabular}
	\caption{The mean and standard deviation for the mean squared error (MSE) and satisfaction ratio (SR) for the Bias Correction (BC) data set and the Family Income (FI) data set. The best results for each setup are shown in bold.}
	\label{tab:ResultsExperiment}
\end{table*}

\section{Conclusion}
\label{sec:Conclusion}

The proposed method, \texttt{CGGD}, enables the use of a conjunction of hard inequality constraints during the learning cycle of neural networks. The method succeeds in fixing the crispness issue by not transforming the constraints, and the vanishing gradient phenomenon by including a rescale factor that is strictly larger than 1 and the lower bound on the norm of the gradient of the loss function. For the purpose of regression on small data sets, the performance in terms of mean squared error and constraints satisfiability was empirically verified on two data sets. The loss was comparable to the other approaches, but the satisfiability of the constraints was always the highest for \texttt{CGGD}.

\begin{footnotesize}

\bibliographystyle{unsrt}
\bibliography{CGGD-ESANN-arXiv}

\end{footnotesize}

\appendix

\section{Proof of Theorem \ref{thm:ConvergenceTheorem}}
In this Appendix a full proof of Theorem \ref{thm:ConvergenceTheorem} is given. The first lemma shows that the updating step can be made smaller outside of the feasible region. This is a necessary property to have convergence to a point on the boarder of the feasible region. Note that the learning rate is chosen in terms of the norm of the gradient of $L$, which is allowed because in order to compute the update step (even without the step size) this value needs to be computed.

\begin{Lemma}\label{lem:DecreasingUpdateStep}
    Let $0<\varepsilon<1$ be fixed and $L:\mathbb{R}^n\to\mathbb{R}$ whose gradient $\nabla L$ is $M$-Lipschitz. Denote by $\overset{\rightarrow}{dir}(C(x))$ the shortest path from the feasible region $FR$ towards $x$ for each $x\in\mathbb{R}^n\setminus FR$. Then the size of the update step outside the feasible region decreases if
        \begin{equation*}
            0 < \eta_{j+1} \leq  \min\left\{\frac{2\eta_j\varepsilon}{25 \eta_j M +10}, \frac{\eta_j}{5}\right\}, \quad \mbox{when } \|\nabla L(x^{(j+1)})\| \geq \varepsilon,
        \end{equation*}
    and 
        \begin{equation*}
            0 < \eta_{j+1} \leq \frac{\eta_j\max\{\varepsilon,\|\nabla L(x^{(j)})\|\}}{5 \varepsilon}, \quad \mbox{when } \|\nabla L(x^{(j+1)})\| < \varepsilon.
        \end{equation*}
\end{Lemma}
\begin{proof}
    Suppose first that $\|\nabla L(x^{(j+1)})\| \geq \varepsilon$. Denote by $m(x):=\max\{\varepsilon,\|\nabla L(x)\|\}$. It is sufficient to show there exists $\eta_{j+1}>0$ such that
        \begin{align}
            \nonumber
            \left\|\eta_{j+1} \nabla L(x^{(j+1)}) +  \right.& \left.\frac{3\eta_{j+1}}{2} \overset{\rightarrow}{dir}(C(x^{(j+1)})) m(x^{(j+1)}))\right\| \\ \nonumber
            &\quad \leq \left\|\eta_{j} \nabla L(x^{(j)}) + \frac{3\eta_{j}}{2} \overset{\rightarrow}{dir}(C(x^{(j)})) m(x^{(j)}) \right\|,
        \end{align}
    for a given $\eta_j>0$. Assume that $m(x^{(j)})=\|\nabla L(x^{(j)})\|$. By construction of the update step, it follows that
        \begin{align*}
            \frac{\eta_j}{2} \|\nabla  L(x^{(j)})\| 
            & \leq \left\|\eta_j \nabla L(x^{(j)}) + \frac{3\eta_j}{2} \overset{\rightarrow}{dir}(C(x^{(j)})) \|\nabla L(x^{(j)})\|\right\| \\
            & \leq \frac{5\eta_j}{2} \|\nabla L(x^{(j)})\|,
        \end{align*}
    for each $j$ and each $x^{(i)}\in\mathbb{R}^n$. Hence it is sufficient to find $\eta_{j+1}>0$ such that
        \begin{equation}\label{eq:lemUpdateStep1}
            \frac{5\eta_{j+1}}{2} \|\nabla L(x^{(j+1)})\| \leq \frac{\eta_j}{2} \|\nabla L(x^{(j)})\|, \qquad \forall x\in\mathbb{R}^n.
        \end{equation}
    Observe that taking $0<\eta_{j+1}< \frac{\eta_j}{5}$ is sufficient in the case where $\|\nabla L(x^{(j+1)})\| \leq \|\nabla L(x^{(j)})\|$. So assume from now on that $\|\nabla L(x^{(j+1)})\| > \|\nabla L(x^{(j)})\|$. By the reverse triangle inequality and $\nabla L$ being $M$-Lipschitz it follows that
        \begin{equation}\label{eq:lemUpdateStep2}
            \|\nabla L(x^{(j+1)}) \| \leq \frac{5\eta_j M + 2}{2} \|\nabla L(x^{(j)})\|,
        \end{equation}
    for each $x^{(j)}\in \mathbb{R}^n\setminus FR$ and $\eta_j>0$. Combining \eqref{eq:lemUpdateStep1} and \eqref{eq:lemUpdateStep2} and rewriting yields that it is sufficient to take
        \begin{equation*}
            \eta_{j+1} \leq \frac{2\eta_j}{25\eta_j M +10},
        \end{equation*}
    which exists because $M$ is a strictly positive fixed real number, and $\eta_j$ is known and non-zero.
    
    In the case where $m(x^{(j)}) = \varepsilon$, the right-hand side of \eqref{eq:lemUpdateStep1} becomes $\frac{\eta_j}{2} \varepsilon$. This implies that it is sufficient to take
        \begin{equation*}
            \eta_{j+1} \leq \frac{2\eta_j \varepsilon}{25\eta_j M +10},
        \end{equation*}
    because $0<\varepsilon < 1$.
    
    At last, when $\|\nabla L (x^{(j+1)})\| < \varepsilon$ then analogous to the previous it follows that it is sufficient to take $\eta_{j+1} \leq \frac{\eta_j}{5\varepsilon}m(x^{(j)})$.
\end{proof}

The second lemma shows that the distance towards the feasible region can be made infinitely small when necessary. This will allow for stating everything locally around a point on the boundary. But first, some notation is introduced for the closed ball around a certain point with a given radius.

\begin{Definition}[Closed ball]
    The closed ball around a point $x$ of radius $\delta$ is given by
        \begin{equation*}
            B[x,\delta] := \{z\mid d(x,z) \leq \delta\}.
        \end{equation*}
\end{Definition}

\begin{Lemma}\label{lem:DistanceToFRSmall}
    Let $\varepsilon>0$ be a fixed constant and denote $m(x):=\max\{\varepsilon, \|\nabla L(x)\|\}$. If $K= d(x^{(j)},FR)$ and $\frac{3}{2} \eta_j m(x^{(j)}) < K$, then $d(x^{(j+1)}, FR) \leq K - \frac{\eta_j}{2} m(x^{(j)})$.
\end{Lemma}
\begin{proof}
    Denote by $y$ the point in $FR$ that gives rise to the vector $\overset{\rightarrow}{dir}(C(x^{(j)}))$. The collection of points that can be obtained from the update step are given by the boundary of $$A:=B\left[x^{(j)} - \frac{3}{2} \eta_j \overset{\rightarrow}{dir} (C(x^{(j)}))m(x^{(j)}), \eta_j \|\nabla L(x^{(j)})\|\right].$$ It is easy to observe that the point furthest away from $y$ is given by $x^{(j)}- \frac{3m(x^{(j)}) }{2} \eta_j \overset{\rightarrow}{dir}(C(x^{(j)})) + \eta_j \|\nabla L(x^{(j)}) \|$ since $\overset{\rightarrow}{dir}(C(x^{(j)}))$ defines the shortest path from $FR$ to $x^{(j)}$. Moreover, when $\|\nabla L(x^{(j)})\| < \varepsilon$ it holds that
        \begin{equation*}
            \frac{3}{2} m(x^{(j)}) - \|\nabla L(x^{(j)}) \| > \frac{1}{2} m(x^{(j)}).
        \end{equation*}
    Hence, it holds that for each $z\in A:$ $$d(x^{(j)} - \frac{\eta_j}{2} \overset{\rightarrow}{dir}(C(x^{(j)})) m(x^{(j)}), y ) \geq d(z,y).$$ The claim follows now from the observation that 
        \begin{align*}
            d(x^{(j+1)}, FR) &\leq d(x^{(j)} - \frac{\eta_j}{2} \overset{\rightarrow}{dir}(C(x^{(j)})) m(x^{(j)}),y)
            \\
            &= K - \frac{\eta_j}{2} m(x^{(j)}).
        \end{align*}
    
\end{proof}

\begin{proof}[Proof of Theorem \ref{thm:ConvergenceTheorem}]
    From Assumption \ref{axiom:NonConvexOptimal}, it follows that if the initialization is inside the feasible region $FR$ and the optimization procedure does not leave $FR$ that convergence occurs. Additionally, if the optimization procedure reaches $FR$ after a finite number of steps and stays in $FR$ for the remaining updates, then convergence occurs as well as a consequence of Assumption \ref{axiom:NonConvexOptimal}.
    
    Observe that combining Lemma \ref{lem:DistanceToFRSmall} and Lemma \ref{lem:DecreasingUpdateStep} and decreasing $\eta_j$ only when $d(x^{(j)}, FR) < \frac{3}{2} \eta_j \max\{\varepsilon, \|\nabla L(x^{(j)}) \|\}$, it follows that the distance to $FR$ can be made arbitrarily small. In particular, this even holds true if no point in $FR$ is found. Therefore, it follows that in this case convergence occurs to some point on the boundary because the size of the update step converges to 0.
    
    Note that when a point $x^{(j)}$ is found in $FR$ for which the next point $x^{(j+1)}$ in the iteration procedure is outside of $FR$, then the distance of $x^{(j+1)}$ to $FR$ is bounded because the update step size can be made bounded and decreasing. Therefore, the size to the feasible region is bounded from above by choosing $\eta_i$ accordingly. Since $\overset{\rightarrow}{dir}(C)$ denotes the closest path to $FR$, it follows that the update step of the constraints does not change the objective value of the point in $FR$ that is closest to the current point outside $FR$. Therefore, if a point $z$ on the boundary $\partial FR$ of the feasible region is optimal, then the gradient of $L$ defines locally $-\overset{\rightarrow}{dir}(C)$ which means that convergence is obtained for a suitable choice of $\eta_j$. To finish the prove it is sufficient to observe that Assumption \ref{axiom:NonConvexOptimal} implies that $L$ is not strictly decreasing when $\left\|x^{(j)}\right\|\xrightarrow{j\to \infty}\infty$ and hence it is not possible for a suitable choice of $\eta_j$ that the objective value keeps decreasing along the boundary of $FR$.
    
\end{proof}

\end{document}